\documentclass[11pt]{article}
\def \shownotes{0}

\usepackage[accepted]{aistats2019}
\usepackage{subcaption}
\usepackage{graphicx}
\usepackage{macro}
\usepackage[numbers]{natbib}
\usepackage{balance}
\usepackage[hidelinks]{hyperref}
\hypersetup{colorlinks=true,linkcolor={DarkRed},citecolor={black},urlcolor={DarkRed}}
\definecolor{DarkRed}{rgb}{0.6,0.2,0.2}

\begin{document}

\twocolumn[
\aistatstitle{Recovery Guarantees for Quadratic Tensors with\\Sparse Observations}
\aistatsauthor{Hongyang R. Zhang \And Vatsal Sharan \And Moses Charikar \And Yingyu Liang}
\aistatsaddress{ Stanford University\\{hongyang@cs.stanford.edu} \And Stanford University\\vsharan@stanford.edu \And Stanford University\\moses@cs.stanford.edu \And UWisconsin-Madison\\yliang@cs.wisc.edu }

]

\begin{abstract}
We consider the tensor completion problem of predicting the missing entries of a tensor. The commonly used CP model has a triple product form, but an alternate family of quadratic models, which are the sum of pairwise products instead of a triple product, have emerged from applications such as recommendation systems. Non-convex methods are the method of choice for learning quadratic models, and this work examines their sample complexity and error guarantee. Our main result is that with the number of samples being only linear in the dimension, all local minima of the mean squared error objective are global minima and recover the original tensor. We substantiate our theoretical results with experiments on synthetic and real-world data, showing that quadratic models have better performance than CP models where there are a limited amount of observations available.
\end{abstract}

\section{Introduction}

Tensors provide a natural way to model higher-order data \cite{AGDST14,KB09,nickel2011three}.
They have applications in recommendation systems \cite{pairwise10,rendle14}, knowledge base completion \cite{bordes13,garcia16}, predicting geo-location trajectories \cite{WZX14} and so on.
Most tensor datasets encountered in the above settings are not fully observed.
This leads to tensor completion, the problem of 
predicting
the missing entries, %
given a small number of observations from the tensor \cite{KB09}.
In order to recover the missing entries, it is important to consider the data efficiency of the tensor completion model.

One of the most well-known tensor models is the {\small CANDECOMP/PARAFAC} or CP decomposition \cite{KB09}.
For a third-order tensor, the CP model will express the tensor as the sum of rank 1 tensors, i.e., the tensor product of three vectors.
The tensor completion problem of learning a CP decomposition has received a lot of attention recently \cite{JO14,MS16}.
It is commonly believed that reconstructing a third-order $d$ dimensional tensor in polynomial time requires $\Theta(d^{3/2})$ samples \cite{JO14,BM16a}.
This is necessary even for low-rank tensors, where $\Theta(d)$ samples are information-theoretically sufficient for recovery.
The sampling requirement of CP decomposition limits its representational power for sparsely observed tensors in practice \cite{garcia16}. More precisely, we mean that the number of observed tensor entries is only  of order $O(d)$.
%While regularization may be helpful when there are limited observations, adding strong regularization will also hurt the optimization performance.

On the other hand, an alternative family of quadratic tensor models has emerged from applications in recommendation systems
\cite{pairwise10} and knowledge base completion \cite{N17}.
The {\it pairwise interaction} model has demonstrated strong performance for the personalized tag recommendation problem \cite{rendle14}.
In this model, the $(i,j,k)$ entry of a tensor is viewed as the sum of pairwise inner products:
$\inner{x_i}{y_j} + \inner{x_i}{z_k} + \inner{y_j}{z_k}$, where $x_i, y_j, z_k$ correspond to the embedding of each coordinate.
As another example, the translating embedding model \cite{bordes13} for knowledge base completion can be (implicitly) viewed as solving tensor completion with a quadratic model.
Suppose that $x, z$ are the embedding of two entities, and $y$ is the embedding of a relation. Then the smaller $\norm{x + y - z}^2$ is, the more likely $x, z$ are related by $y$.
Concretely, a quadratic tensor is specified by:
\begin{align*}
	T_{i, j, k} = \sum_{l = 1}^{r} \kappa(A_{i,l}, B_{j, l}, C_{k, l}), \forall~ 1 \le i, j, k \le d,
\end{align*}
where $A, B, C \subseteq \real^{d \times r}$ correspond to the embedding vectors, and $\kappa : \real^3 \rightarrow \real$ denotes a quadratic function. Both the pairwise interaction model and the translational embedding model correspond to specific choices of $\kappa$.

It is known that for the case of pairwise interaction tensors, the linear (in dimension) number of samples is enough to recover the tensor via convex relaxation \cite{CLK13}.
However, in practice, non-convex methods are the predominant method of choice for training quadratic models.
This is because non-convex methods, such as alternating minimization and gradient descent, are more scalable to handle large datasets.
Despite the practical success, it has been a challenge to analyze the performance of non-convex methods theoretically.
In this work, we present the first recovery guarantee of non-convex methods for learning quadratic tensors.
Besides the motivation of quadratic tensors, our work joins a line of recent work to understand further when local methods can lead to globally optimal solutions in non-convex low-rank problems \cite{GLM16,BVB16,GJZ17}.
Our results show that quadratic tensor completion enjoys the property that all local minima are global minima in its non-convex formulation.

\paragraph{Main Results.}
Assume that we observe $m$ entries of $T$ uniformly at random.
Denote the set of observed entries as $\Omega$.
Consider the natural least squares minimization problem $f(X, Y, Z)$:
{\begin{align*}
	 \sum_{(i, j, k) \in \Omega} \bigbrace{\sum_{l=1}^R \kappa(X_{i,l}, Y_{j,l}, Z_{k,l}) - T_{i, j, k}}^2 
	+ Q,
\end{align*}}%
where $Q = Q(X, Y, Z)$ includes weight decay and other regularizers.
See Section \ref{sec_main} for the precise definition.
Note that $f(X, Y, Z)$ is, in general, non-convex since it generalizes the matrix completion setting when $\kappa(X_{i, l}, Y_{j,l}) = X_{i,l} Y_{j,l}$.
We show that as long as $R \ge 2\sqrt{m}$, all local minima can be used to reconstruct the ground truth $T$ accurately.

\begin{theorem}\label{thm_intro}
	Assume that for all $1 \le i \le d$, $\norm{e_i^{\top} A}, \norm{e_i^{\top} B}, \norm{e_i^{\top} C} \le \sqrt{\mu r / d}$.
	Let $\varepsilon$ be the desired accuracy and $m = \Theta(d r^4 \mu^4 (\log d) / {\varepsilon^2})$.
	For the regularized objective $f$, as long as $R \ge 2\sqrt m$, then all local minimum $V$ of $f$ can be used to reconstruct $\hat{T} \subseteq \real^{d \times d \times d}$ such that
	\begin{align} \frac 1 {d^3} \sum_{1\le i,j,k \le d} \left({\hat{T}_{i,j,k} - T_{i,j,k}}\right)^2 \lesssim \frac {\varepsilon} {d^2}. \label{eq_l1_norm}\end{align}
\end{theorem}
In the incoherent setting, when $\mu$ is a small constant,
the tensor entries are on the order of $1/d$.
Our results imply that the average recovery error is on the order of $\varepsilon / d$.
Hence we recover most tensor entries up to less than $\varepsilon$ relative error.
Our result applies to any quadratic tensor, whereas the previous result on convex relaxations only applies to pairwise interaction tensors \cite{CLK13}.
An additional advantage is that our approach does not require the low-rank assumption for recovery, we only need $r$ in Theorem \ref{thm_intro} to be small, where $r$ is upper bounded by the rank $R$. We also note that the $r^4$ dependence of the sample complexity on $r$ for our results is comparable to recent results for non-convex methods for matrix completion \citep{GJZ17}.

Our technique is based on over-parameterizing the search space to dimension $R = \Theta(\sqrt{m})$ (the $R = \Theta(\sqrt{m})$ dependence on over-parameterization is comparable to previous analyses for low-rank Burer-Monteiro formulations \citep{BVB16,DL18}).
We show that for the training objective, there is no bad local minimum after over-parameterization.
Hence any local minima can achieve small training errors.
The regularizer $Q$ is then used to ensure that the generalization error to the entire tensor is small, provided with just a linear number of samples from $\Omega$.
Since the result applies to any local minimum, it has implications for any non-convex method conceptually, such as alternating least squares and gradient descent.
%Another remark is that the $\ell_1$ norm guarantee stated in \eqref{eq_l1_norm} can also be suitably adjusted if we were to measure the performance in the $\ell_2$ norm.
%The scale of the $\ell_2$ norm would be different, but one can show the same sample complexity bound in this norm instead.

\paragraph{Experiments.} We substantiate our theoretical results with experiments on synthetic and real-world tensors.
Our synthetic experiments validate our theory that non-convex methods can recover quadratic tensors with a linear number of samples.
Our real-world experiments compare the CP model and the quadratic model solved using non-convex methods on two real-world datasets.
The first dataset consists of $10$ million ratings over time (Movielens-10M).
The task is to predict movie ratings by completing the missing entries of the tensor. We found that the quadratic model outperforms CP-decomposition by $10\%$.
The second dataset consists of a word tri-occurrence tensor comprising the most frequent 2000 English words.
We learn word embeddings from the tensor using both the quadratic model and the CP model
and evaluate the embeddings on standard NLP tasks.
The quadratic model is $20\%$ more accurate than the CP model. These results indicate that the quadratic model is better suited to sparse, high-dimensional datasets than the CP model, and we hypothesize that this stems from its better data efficiency.

\paragraph{Summary.}
In conclusion, we show that provided with just a linear number of samples from a quadratic tensor, we can recover the tensor accurately using any local minimum of the natural non-convex formulation.
Empirically, the quadratic models enjoy superior performance when solved with the non-convex formulation compared to the CP model.
Together, they indicate that the quadratic model is a suitable tensor model in practical settings with limited data.

\paragraph{Notations.}
Given a positive integer $d$, let $[d]$ denote the set of integers from $1$ to $d$.
For a matrix $X \in \dimMatrix{d_1 \times d_2}$, let $X_i$ denote
the $i$-th row vector of $X$, for any $i \in [d_1]$.
We use $X \gematrix \zeroMatrix$ to denote that $X$ is positive semi-definite.
Denote by $\cS_{d}$ as the set of symmetric matrices of size $d$ by $d$.
Denote by $\cS_{d}^+$ as the set of $d$ by $d$ positive semidefinite matrices.
Let $\norm{\cdot}$ denote the Euclidean norm of a vector and spectral norm of a matrix.
Let $\normFro{\cdot}$ denote the Euclidean norm of a matrix.
Let $\norm{\cdot}_1$ denote the $\ell_1$ norm of a matrix or tensor, i.e., the sum of the absolute value of every entry.
For two matrices $A,B$, we define the inner product $\langle A,B \rangle =\text{Tr}(AB^T)$.
For three matrices $X, Y, Z \in \real^{d \times d'}$,
denote by $[X, Y, Z] \in \real^{3d \times d'}$ as the three matrices stacked vertically.

Given an objective function $f: \real^d \rightarrow \real$, we use $\nabla f(U)$ to denote the gradient of $f(U)$, and $\nabla^2 f(U)$ to denote the Hessian matrix of $f(U)$, which is of size $d$ by $d$.
We denote $f(x) \lesssim g(x)$ if there exists an absolute constant $C$ such that $f(x) \le C g(x)$.

\paragraph{Organization.} The rest of the paper is organized as follows.
In Section \ref{sec_setup}, we define the quadratic model more formally and review related work.
In Section \ref{sec_main}, we present our theoretical results.
In Section \ref{sec_exp}, we experimentally evaluate the non-convex formulation for solving quadratic models.
We conclude our paper in Section \ref{sec_conc}.
In Appendix \ref{sec_proof_rade}, we give missing proofs left from the main text.
In Appendix \ref{sec_add_exp}, we describe several additional experiments.
\section{Preliminaries}\label{sec_setup}
	We now define the quadratic model more formally with examples. Recall that $T \in \dimMatrix{d \times d \times d}$ is a third order tensor,
	composed by a quadratic function over three factor matrices $A, B, C \subseteq \real^{d \times r}$.\footnote{We assume that the three dimensions all have size $d$ in order to simplify the notations.
	It is not hard to extend our results to the more general case when different dimensions have different sizes.
	Also, we will focus on third-order tensors for ease of presentation -- it is straightforward to extend the quadratic model to higher orders.}
	%In the introduction we defined $\kappa$ as a function on real values, 
    We now define $\kappa: \real^{d} \times \real^{d} \times \real^{d} \rightarrow \real$ as a function, supported on the cross product between three input vectors.
	More specifically,
	\begin{align*}
		T_{i, j, k} &= \kappa(A_i, B_j, C_k) \\
		&= \inner{[A_i, B_j, C_k]}{K [A_i, B_j, C_k]}.
	\end{align*}
 \begin{itemize}
	\item Recall that $[A_i, B_j, C_k]$ is a $3\times d$ matrix with the $i$-th, $j$-th, $k$-th rows of $A,B,C$ stacked.
 \item The kernel matrix $K \in \real^{3 \times 3}$ encodes the similarity/dissimilarity represented by $\kappa$ between the input vectors.  Different choices of $K$ represent different quadratic models; for example when $K=I$, 
 \[ T_{i, j, k}=\norm{A_i}^2+\norm{B_j}^2+\norm{C_k}^2. \]
 We assume that $K$ is a symmetric matrix without loss of generality since we can always symmetrize $K$.
 \end{itemize}
	We now describe two quadratic models which are commonly used in the literature.

	\begin{example}
		The Pairwise Interaction Tensor Model \cite{pairwise10}  is proposed in the context of tag recommendation, e.g., suggesting a set of tags that a user will likely use for an item.
		The Pairwise Model scores the triple $(i, j, k)$ with measure:
		\begin{align*}\label{eq_pairwise}
			T_{i, j, k} = \inner{A_i}{B_j} + \inner{A_i}{C_k} + \inner{B_j}{C_k}.
		\end{align*}
	\end{example}
For this model, the kernel matrix $K$ has $1/2$ on all off-diagonal entries and 0 on the diagonal entries. In the tag recommendation setting,  $A_i, B_j$, and $C_k$ correspond to embeddings for the $i$th user, $j$th item, and $k$th tag, respectively. The pairwise interaction model models two-way interactions between users, items, and tags to predict if user $i$ is likely to use  tag $k$ for item $j$. 

	\begin{example}
		The Translational Embedding Model (a.k.a TransE) \cite{bordes13} is well studied in the knowledge base completion problem, e.g., inferring relations between entities.
		The TransE model scores a triple $(i, j, k)$ with
		\[ T_{i, j, k} = \norm{A_i + B_j - C_k}^2. \]
		Intuitively, the smaller $T_{i, j, k}$ is,
		the more likely that entities $i$ and $k$ will be related by relation $j$.
	\end{example}
	
	The idea here is that if adding the embedding for \emph{Italy} to the embedding for the {\it capital of} relationship results in a vector close to the embedding for \emph{Rome}, then \emph{Rome} and \emph{Italy} are likely to be linked by the {\it capital of} relation.

\section{Related Work}

We first review existing approaches for analyzing non-convex low-rank problems.
One line of work focuses on the geometry of the non-convex problem and show that as long as the current solution is not optimal, then a direction of improvement can be found \cite{GLM16,BVB16,GJZ17}.
There are a few technical difficulties in applying this approach to our setting.
One difficulty is asymmetry --- our setting requires recovering three sets of different parameters.
Existing analysis of alternating least squares does not seem to apply because of the asymmetry as well \cite{LYR16}.
The second difficulty is that there exist multiple factor matrices which correspond to the same quadratic tensor in our setting.
Hence it is not clear which factor matrices the gradient descent algorithms converge to.
A second line of work builds on a connection between SDPs and their Burer-Monteiro low-rank formulations \cite{BVB16}.
Our proof expands on the intuition from this line of work; we deal with a few additional complexities in the tensor completion problem.
Recent work has applied this connection to analyzing over-parameterization in one hidden layer neural networks with quadratic activations \cite{DL18}.
Our techniques are inspired by this work while considering the incoherence of the factor matrices and adding the incoherence regularizer to our setting \cite{GLM16,BVB16}.
We refer the reader to Section \ref{sec_main} for more technical details.

Next, we review related works for tensor completion.
One approach is to flatten the tensor into a matrix or treat each slice of the tensor as a low-rank matrix individually
and then apply matrix completion methods \cite{GRY11}.
There are other models such as RESCAL \cite{nickel2011three}, Tucker-based methods \cite{KB09} etc.
We refer the reader to a recent survey for more information \cite{SGCH17}.
From a matrix factorization perspective,
%\citet{C15,SLLS16} propose a latent variable model approach for collaborative filtering.
%Every user and item is represented by a latent representation, and there is a kernel function that determines a user's preference over an item,
%given their latent representations.
%\citet{SLLS16} characterize general conditions under which it is possible to learn
%the latent representation.
\citet{GL16} consider computing graph embeddings for the link prediction problem. They consider different operators to learn edge features and found the inner product to perform the best in their experiments.

%\paragraph{Algorithms for computing CP decompositions:}
There has been a line of recent research on provably recovering random and smoothed tensors using algorithms such as gradient descent and alternating minimization \cite{anandkumar2014guaranteed,GM17},
which are by far the most popular algorithms for CP decomposition \cite{KB09}.
The {sum-of-squares} framework has also emerged as a powerful theoretical tool for understanding tensor decompositions.
A long line of work goes back to using simultaneous diagonalization and higher-order SVD-based approaches \cite{de2006link}, and recent work has attempted to make these algorithms more noise robust and scalable \cite{BCMV14,SRT15,KCL15,colombo2016tensor}.
Recently, there has been significant interest in developing more computationally efficient and scalable algorithms for CP decomposition \cite{wang2015fast,SWZ16}.
For other tensor decomposition models, such as Tucker decomposition, see \citet{KB09}.

\section{Recovery Guarantees}
\label{sec_main}
In this section, we consider the recovery of quadratic tensors under partial observations.
Recall that we observe $m$ entries uniformly randomly from an unknown tensor $T$.
Let $\Omega \in [d]^3$ denote the indices of the observed entries.
Given $\Omega$, our goal is to recover $T$ accurately. We first review the definition of local optimality conditions.

\begin{definition}
	(Local minimum) Suppose that $U$ is a local minimum of $f(U)$, then we have that
	$\nabla f(U) = \zeroMatrix$ and $\nabla^2 f(U) \gematrix \zeroMatrix$.
\end{definition}

We focus on the following non-convex least squares formulation with variables $X, Y, Z$, which model the unknown parameters.
In this setting, we assume that $\kappa$ is already known.
This is without loss of generality since our approach applies to the case when $\kappa$ is unknown using the same proof technique.
{%
\begin{align*}
	&\min_{X,Y,Z \subseteq \real^{d \times R}} g(X,Y,Z) = \\& \frac 1 m \sum_{(i, j, k) \in \Omega} \bigbrace{\sum_{l=1}^R \kappa(X_{i,l}, Y_{j,l}, Z_{k,l}) - T_{i, j, k}}^2 \\
	+ & \lambda_1 (\normFro{X}^{2} + \normFro{Y}^2 + \normFro{Z}^2) 
	+  \lambda_2 \sum_{i=1}^d q_{\alpha}(\norm{e_i^{\top} U}) \\
	+ & \inner{[X;Y;Z]}{C  [X;Y;Z]}.
\end{align*}}%
Let us unpack the above function.
The first term corresponds to the natural MSE over $\Omega$.
Next we have $q_{\alpha}(x) = (\abs{x} - \sqrt{\alpha})^4 \ind{x \ge \sqrt{\alpha}}$.
The role of $q_{\alpha}(x)$ is to penalize any row of $X, Y, Z$ whose norm is higher than $\sqrt{\alpha}$, the desired amount from our assumption.
It is not hard to verify that $q_{\alpha}(x)$ is twice differentiable.

Last, $C \subseteq \cS_{3d}^+$ is a random PSD matrix with the spectral norm at most $\lambda_1$.
One can view $C$ as a small perturbation on the loss surface.
This perturbation will be essential to smooth out unlikely cases in our analysis, as we will see later.
Our main result is described below.

\begin{theorem}[Restatement]\label{thm_main}
	Let $T^{\star} \subseteq \real^{d \times d \times d}$ be a quadratic tensor defined by factors $A^{\star}, B^{\star}, C^{\star} \subseteq \real^{d \times r}$ and a quadratic function $\kappa$.
	Assume that
	\[ \norm{e_i^{\top} A^{\star}}, \norm{e_i^{\top} B^{\star}}, \norm{e_i^{\top} C^{\star}} \le \sqrt{\alpha}, \forall~ 1 \le i \le d. \]
	We are given a uniformly random subset of $m$ entries $\Omega \subseteq [d]^3$ from $T^{\star}$.
	Let $m \gtrsim d (\log d) / \varepsilon^2$ and $R \ge \sqrt{2m + 2d}$.
	Under appropriate choices of $\lambda_1$ and $\lambda_2$,
	for any local minimum $X,Y,Z$ of $g$, with high probability over the randomness of $\Omega$ and $C$,
	for $\hat{T}_{i, j, k} = \sum_{l=1}^R \kappa(X_{i,l}, Y_{j,l}, Z_{k,l})$,
	we have:
	\[ \frac 1 {d^3} \bignormFro{ {\hat{T} - T^{\star}}}^2 \lesssim \alpha^2 \varepsilon, \]
\end{theorem}
Note that Theorem \ref{thm_intro} is the same as Theorem \ref{thm_main} by setting $\alpha = {\mu r/d}$ as well as the corresponding value of $m$ and $R$ in Theorem \ref{thm_main}.

For a concrete example of the recovery guarantee, suppose that $A^{\star}, B^{\star}, C^{\star}$ are all sampled independently from $\cN(0, 1 /  d)$.
In this case, one can verify that $\alpha \lesssim r (\log d) / d$.
Hence when $m \gtrsim d r^4 \log^5 d / \varepsilon^2$, the average recovery error is at most $O(\varepsilon / d^2)$.
Since every entry of $T^{\star}$ is on the order of $1/d$ based on the quadratic model,
the theorem shows most tensor entries are accurately recovered up to a small relative error.

\paragraph{Proof overview.}
Next, we give an overview of the technical insight.
The first technical complication of analyzing such a $g(X, Y, Z)$ is that the three factors are asymmetric.
Therefore to simplify the analysis, we first reduce the problem to a symmetric problem by viewing the search space as $[X; Y; Z] \in \real^{3d \times r}$ instead.
We then show that all local minima of $g(X, Y, Z)$ are global minima.

Here is where we crucially use the random perturbation matrix $C$ -- this is necessary to avoid a zero probability space which may contain non-global minima.
While this idea of adding a random perturbation is inspired by the work of Du and Lee \cite{DL18}, adapting to our setting is novel and requires careful analysis.

In the last part, we use the regularizer of $g$ to argue that all local minima are incoherent and their Frobenius norms are small.
Based on these two facts, we use Rademacher complexity to analyze the generalization error.
We now go into the details of the proof.

%\vspace{-8pt}
\paragraph{Local optimality.} Before proceeding, we introduce several notations.
Denote by $U^{\star} = [A^{\star}, B^{\star}, C^{\star}] \subseteq \real^{3d \times r}$ as the three factors stacked vertically.
Let $X^{\star} = U^{\star} {U^{\star}}^{\top}$.

For each triple $t = (i, j, k) \in [d]^3$, denote by $A_t \subseteq \real^{3d \times 3d}$ as a sensing matrix such that $\inner{A_t}{X^{\star}} = T^{\star}_{i, j, k}$.
Specifically, we have that $A_t$ restricted to the row and column indices $i, j+d, k+2d$ is equal to $K$ (the kernel matrix of $\kappa$), and $0$ otherwise.

We can rewrite $g(X, Y, Z)$ more concisely.
\begin{align*}
	f(U) = &\frac 1 m \sum_{t \in \Omega} \inner{A_t}{UU^{\top} - X^{\star}}^2 + \lambda_1 \normFro{U}^2\\
	&+ \lambda_2 \sum_{i=1}^{3d} q_{\alpha}(\norm{e_i^{\top} U})
	+ \inner{C}{UU^{\top}},
\end{align*}
where $U = [X;Y;Z] \subseteq \real^{3d \times R}$.
We will use the following Proposition in the proof.

\begin{proposition}[Proposition 4 in Bach et al. \cite{BMP08}]\label{prop_deficient}
		Let $g$ be a twice differentiable convex function over $\cS_d^+$.
		If the function $h: U \rightarrow g(UU^{\top})$ defined over $U \subseteq d \times d'$ has a local minimum at a rank deficient matrix $V$, then $VV^{\top}$ is a global minimum of $g$.
\end{proposition}

Now we are ready to show that there are no bad local minima in the loss landscape of $f(U)$.
\begin{lemma}\label{lem_local}
	In the setting of Theorem \ref{thm_main}, with high probability, any local minimum $U$ of $f(\cdot)$ is a global minimum.
\end{lemma}
%\vspace*{-14pt}
\begin{proof}
	We will show that $\rank(U) < R$, hence by Proposition \ref{prop_deficient}, $U$ is a global minimum of $f(U)$.
	Assume that $\rank(U) = R$.
	By local optimiality, $\nabla f(U) = \zeroMatrix$, we obtain that:
	\begin{align}
		& \left(\sum_{t \in \Omega} z_t A_t
		+ \sum_{i = 1}^d w_i e_i e_i^{\top}
		+ \lambda_1 \id + C \right) U = \zeroMatrix, \label{eq_perturb} \\
		&\quad\quad \text{where}~ w_i = \frac{4\lambda_2 (\norm{e_i^{\top} U} -\sqrt{\alpha})^3} {\norm{e_i^{\top} U}} \ind{\norm{e_i^{\top} U} \ge \sqrt{\alpha}}, \nonumber\\
		&\quad\quad \text{and}~ z_t = \frac 2 m \inner{A_t}{UU^{\top} - X^{\star}}. \nonumber
	\end{align}
	Denote by
	\begin{align*}
		& M(w, z) = \sum_{i=1}^d w_i e_i e_i^{\top} + \sum_{t \in \Omega} z_t A_t, \text{ and } \\
		& \cA = \Big\{X - M(w, z) - \lambda_1 \id: X \in \cS_{3d}, XU = \zeroMatrix,
			\\
		 & \;\qquad w \in \real^d, z \in \real^m\Big\}.
	\end{align*}
	In the above definition, $X$ is a symmetric matrix in the null space of $U$ -- recall that $A_t$ is symmetric for any $t \in [d]^3$.
	The set $\cA$ is a manifold and  $C \in \cA$ by the gradient condition.

	Since the rank of the null space is $3d - R$, the dimension of such $X$ is $\frac{3d(3d+1)} 2 - \frac{R(R+1)} 2$.
	Together with $w$ and $z$, we have that the dimension of $\cA$ is
	\[ \frac{3d(3d+1)} 2 - \frac{R(R+1)} 2 + m + d. \]
	We have assumed that $R \ge \sqrt{2m + 2d}$.
	Hence the dimension of $\cA$ is strictly less than $\frac{3d(3d+1)} 2$.
	However, the probability that a random PSD matrix $C$ falls in such a set $\cA$ only happens with probability zero.
	Hence with high probability, the rank of $V$ is less than $R$.
	The proof is complete.
\end{proof}

\paragraph{Rademacher complexity.}
Next, we bound the generalization error using Rademacher complexity.
We first introduce some notations.
For any $S \subseteq [d]^3$, $X \subseteq \cS_{3d}^+$, denote by
\[ \cL_{S}(X) = \frac 1 {\abs{S}} \sum_{t \in S} \inner{A_t}{X - X^{\star}}^2. \]
Let $\cG$ denote the set of matrices as follows.
{\begin{align*}
	\cG = \bigset{X \in \cS_{3d}^+ : \tr(X) \le 6d\alpha, X_{i, i} \le 2\alpha \forall  i \in [d] }
\end{align*}}%
Denote by $\cT$ the set of quadratic tensors constructed from matrices in $\cG$.
\begin{align*}
	\cT \define \Big\{T \in \real^{[d]^3},& \text{ where } T_{i, j, k} = \inner{A_t}{X}, \forall\;\\
    &t = (i, j, k) \in [d]^3 : X \in \cG\Big\}
\end{align*}

We bound the Rademacher complexity of $\cT$ in the following Lemma.

\begin{lemma}\label{lem_rade}
	In the setting of Theorem \ref{thm_main}, we have that
	{\begin{align*} & \exarg{\Omega}{\sup_{X \in \cG} \bigabs{\cL_{\Omega}(X) - \cL_{[d]^3}(X)}} \\
 \lesssim& \alpha^2 \sqrt{\frac d m + \frac {d^2 \log d}{m^2}}. \end{align*}}%
%	\[ \sup_{X \in \cG} \abs{\cL_{\Omega}(X) - \cL_{[d]^3}(X)} \le c \norm{K}_1 d \alpha^2 \varepsilon. \]
\end{lemma}

The proofs for Lemma \ref{lem_rade} as well as Theorem \ref{thm_main} are deferred to the Appendix
-- The latter follows by combining Lemma \ref{lem_local} and Lemma \ref{lem_rade}.%

\paragraph{Polynomial time algorithms.}
Next, we discuss algorithms for minimizing $g(X, Y, Z)$.
Apart from the gradient descent algorithm,
minimizing $g(\cdot)$ can also be solved via {\it alternating least squares (ALS)}, because on fixing $B$ and $C$, $g(\cdot)$ is an $\ell_2$ regularized least-squares problem over $A$;
similarly for $B$ and $C$.
Hence ALS alternatively solves $\ell_2$ regularized least-squares problems and terminates after a predefined maximum number of iterations or if the error does not decrease in an iteration.
Each iteration involves at most $O(r d^3)$ computations but can be substantially faster if the original tensor is sparse, in which case the computational complexity essentially only depends on the sparsity of the original tensor.
We will validate the performance of gradient descent and ALS for synthetic data in Section \ref{sec_synthetic}.
It is an interesting open question to analyze the convergence of gradient descent or ALS for quadratic tensors.

Lastly, minimizing $g(\cdot)$ can be solved via convex relaxation methods as follows.
\begin{eqnarray*}
	h(\Omega, y) \define & \min &~ \frac 1 m \sum_{t \in \Omega} (\inner{A_t}{X} - y_{t})^2 \\
			&	&~ + \inner{C}{X} + \lambda_1 \sum_{i=1}^{3d} q_{\alpha}(X_{i, i}) \\
	& \mbox{s.t.} & \tr(X) \le 3d \alpha, \\
	&						& X \gematrix \zeroMatrix,
\end{eqnarray*}
where we recall that $\inner{A_t}{X}$ corresponds to the $t = (i,j,k)$-th entry of the quadratic tensor defined by $X$.
Note that the objective function is convex, and the feasible region is convex and bounded from above. Hence, the problem can be solved in polynomial time (see, e.g., Bubeck \cite{B15}).
Combining Lemma \ref{lem_rade} and the proof of Theorem \ref{thm_main}, we obtain the following recovery guarantee for the above convex relaxation method.

\begin{corollary}\label{thm_runtime}
	Let $T^{\star} \subseteq \real^{d \times d \times d}$ be a quadratic tensor defined by factors $A^{\star}, B^{\star}, C^{\star} \subseteq \real^{d \times r}$ and a quadratic function $\kappa$.
	Assume that \[ \norm{e_i^{\top} A^{\star}}, \norm{e_i^{\top} B^{\star}}, \norm{e_i^{\top}C^{\star}} \le \sqrt{\alpha}, \forall~ 1\le i \le d. \]
	Let $\Omega$ be a set of $m$ entries sampled uniformly at random from $T^{\star}$ and $y \subseteq \real^m$ be the entries of $T^{\star}$ corresponding to the indices of $\Omega$.
	When $m \gtrsim d (\log d) / \varepsilon^2$, then solving $h(\Omega, y)$ using convex optimization methods can return a solution $X \subseteq \real^{3d \times 3d}$ in time $\poly(d, r, \alpha)$.
	And $X$ can be used to reconstruct
	$\hat{T}_{i, j, k} = \inner{A_{i, j, k}}{X}$ for all $1 \le i, j, k \le d$ satisfying:
	\[ \frac 1 {d^3} \normFro{\hat{T} - T} \lesssim \alpha^2 \varepsilon. \]
\end{corollary}

\paragraph{Discussions.}
One interesting question is for Theorem \ref{thm_main}, whether the number of parameters can be reduced from $R = \Theta(\sqrt{m})$ to $R = \tilde{O}(\poly(r))$, which does not grow polynomially with dimension.
Here we describe an interesting connection between the above question and the notion of matrix rigidity \cite{GT16}.
Concretely, we ask the following question.

\begin{question}\label{question}
	Let $U \subseteq \real^{d \times d}$ be a random matrix where every entry is sampled independently randomly from a fixed distribution (e.g., standard Gaussian).
	Denote by $X = UU^{\top}$.
	Suppose that we are allowed to arbitrarily change $m = d k$ entries of $X$ and obtain $X'$.
	In other words, $X$ and $X'$ differ by at most $m$ entries.
	What is the minimum possible rank of $X'$?
\end{question}

One can obtain $X'$ by removing $k$ rows from $X$.
Hence the minimum rank would be at most $d - k$.
If the answer to the above question is $d - \Theta(k)$, then Theorem \ref{thm_main} would be true for $R = O(\poly(r))$.
To see this, in the proof of Lemma \ref{lem_local}, we can use $UU^{\top}$ as the random perturbation and scale down the perturbation matrix so that its spectral norm is under the desired threshold.
Then, since there are at most six non-zero entries in $A_t$, for any $t \in \Omega$.
Hence overall, the following equation
$ \sum_{t \in \Omega} z_t A_t + \sum_{i=1}^{3d} w_i e_i e_i^{\top} + \lambda_1 \id$ changes the perturbation matrix $C$ in at most $6m + 3d$ entries (c.f. Equation \eqref{eq_perturb}).
If indeed the rank of $C$ is at least $3d - \Theta(\frac m d)$, then we can set $R = \Theta(\frac m d)$ to obtain the desired result in Lemma \ref{lem_local}.
For accurate recovery we need $m \gtrsim d r^4 (\log d) / \varepsilon^2$, hence $R$ can be reduced to $\Theta(r^4 \log d / \varepsilon^2)$.

Question \ref{question} is equivalent to asking what is the rigidity of a random PSD matrix.
It turns out that understanding the rigidity of random matrices is technically challenging and there is an ongoing line work to further improve our understanding in this area.
We refer the interested reader to the work of Goldreich and Tal \cite{GT16} for details.

\section{Experiments}\label{sec_exp}

In this section, we describe our experiments on synthetic data and real-world data.
For synthetic data, we validate our theoretical results and show that the number of samples needed to recover
the tensor only grows linearly in the dimension using two non-convex methods -- gradient descent and alternating least squares (ALS).
We then evaluate the quadratic model using non-convex methods on real-world tasks in two diverse domains:
\begin{enumerate}
\item Predicting movie ratings in the Movielens-10M dataset;
\item Learning word embeddings using a tensor of word tri-occurrences;
\item Recovering the hyperspectral image Riberia given incomplete pixel values.
\end{enumerate}

In the Movielens-10M dataset, the quadratic model outperforms CP decomposition by over $10\%$.
In the word embedding experiment, the quadratic model outperforms CP decomposition by more than 20\% across NLP benchmarks for evaluating word embeddings.
In the hyperspectral image experiment, we explicitly vary the fraction of sampled pixels from the image.
We observe that the quadratic model outperforms the CP model when there are limited observations, whereas the CP model excels when more observations are available.\footnote{A link to download our experiment codes is \href{https://drive.google.com/file/d/1OenGfyMhBn7naKROdnwye5aN6m0Xo9yb/view?usp=drive_link}{\underline{here}}.}
Due to limited space, we defer the experiments on word embeddings and image reconstruction to Appendix \ref{sec_add_exp}.

\subsection{Synthetic Data}\label{sec_synthetic}

\begin{figure*}[t]
	\centering
	\begin{subfigure}{0.4\textwidth}
		\includegraphics[width=\textwidth]{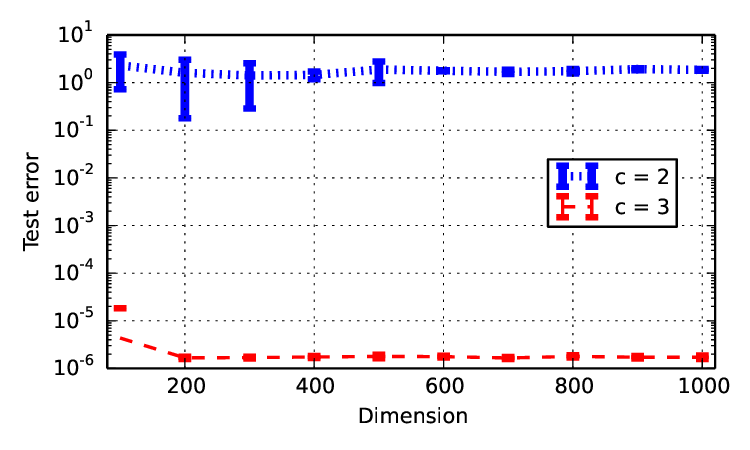}
		\caption{Test error vs. Dimension for ALS}
	\end{subfigure}
	~
	\begin{subfigure}{0.4\textwidth}
		\includegraphics[width=\textwidth]{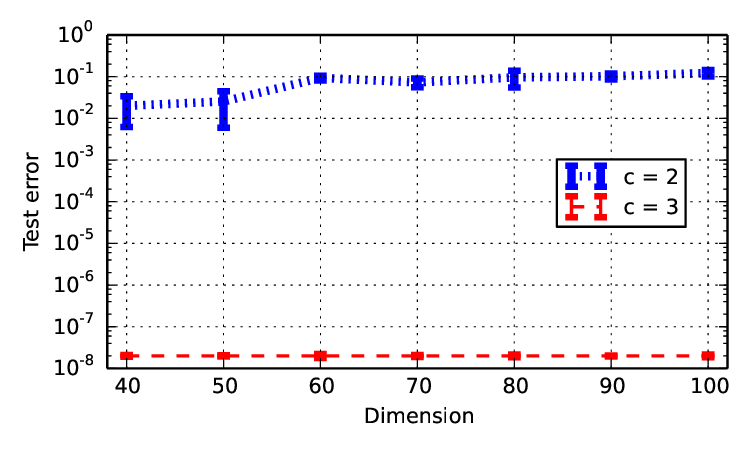}
		\caption{Test error vs. Dimension for SDP}
	\end{subfigure}
	\caption{ALS and SDP require $2dr$ to $3dr$ samples to recover a quadratic tensor with random factors, showing that their sample complexity is $O(d)$. Here $r = 5$ and number of samples $m = cdr$.}
	\label{fig_als_sample}
\end{figure*}

{%\small
\begin{table*}[!ht]
	\centering
	\begin{tabular}{@{}lllll@{}}
		\toprule
		Algorithm            & \multicolumn{2}{c}{Sampling rate $p = 0.2$}                         & \multicolumn{2}{c}{Sampling rate $p = 0.8$}                         \\ \midrule
		& \multicolumn{1}{c}{Rank $r$=10} & \multicolumn{1}{c}{ Rank $r$=20} & \multicolumn{1}{c}{Rank $r$=10} & \multicolumn{1}{c}{Rank $r$=20} \\
		Matrix model  &   $0.872 \pm 0.004$   &     $0.947 \pm 0.002$     &   $0.665 \pm 0.003$   &   $0.667 \pm 0.001$                \\
		CP model  &   $1.068 \pm 0.087$   &     $1.141 \pm 0.054$     &   $0.719 \pm 0.010$   &   $0.705 \pm 0.002$                \\
		Quadratic model  &   $\mathbf{0.798 \pm 0.003}$   &     $\mathbf{0.772 \pm 0.003}$     &   $\mathbf{0.642 \pm 0.002}$   &   $\mathbf{0.638 \pm 0.002}$                \\
		\bottomrule
	\end{tabular}
	\caption{Results for the Movielens-10M dataset for varying sampling rates corresponding to different training and test splits and factorization ranks. The quadratic model yields the best results across all settings, with the larger gap at lower sampling rates.}
	\label{table:movie}
\end{table*}
}

Both gradient descent and ALS are common paradigms for solving non-convex problems, and hence our goal in this section is to evaluate their performances on synthetic data. The ALS approach minimizes the mean squared error objective by iteratively fixing two sets of factors, and then solving the regularized least squares problem on the third factor. In addition, we also evaluate a semidefinite programming-based approach that solves a trace minimization problem, similar to the approach in \citet{CLK13}. 

We now describe our setup. Let $A, B, C \in \dimMatrix{d \times r}$, where every entry is sampled independently from a standard normal distribution.
We sample a uniformly random subset of $m$ entries from the quadratic
tensor $T = \tensor(A, B, C)$.
Let the set of observed entries be $\Omega$, and the goal is to recover $T$ given $\Omega$.
We measure test error of the reconstructed tensor $\hat{T}$ as follows:
{
	\begin{align} \sqrt{\frac{ \sum_{(i, j, k) \notin \Omega} (\hat{T}_{i,j,k} - T_{i,j,k})^2 }
	{ \sum_{(i, j, k) \notin \Omega} T_{i, j, k}^2 }}.\label{eq_rmse} \end{align}
}
%\vspace{-8pt}

\textbf{Accuracy.}
We first examine how many samples ALS and the SDP require to recover $T$ accurately.
Let $m = c \times d \times r$. Here, $m$ is the number of samples.
We fix $r = 5$. For each value of $d$ and $c$, we repeat the experiment thrice and report the median value with error bars.
Because ALS is more scalable, we can test on much larger dimensions $d$. \VS{Should we actually have some representative timing numbers to give readers a sense of how much more efficient ALS is for the problem?}
Fig. \ref{fig_als_sample} shows that the sample complexity of both the SDP and  ALS is 
between $2dr$ to $3dr$.
When $m = 2dr$, both the SDP and ALS fail to recover $T$, but given $m = 3dr$ samples, they can recover $T$ very accurately. ALS also converges within 30 iterations across our experiments (Fig. \ref{fig_scale} in the Appendix shows how the error decays with the iteration). This makes ALS highly scalable for solving the problem on large tensors. We also repeat the same experiment for gradient descent (Section \ref{sec:gd_exp} in the Appendix) and show that it also has linear sample complexity---though the constants seem to be worse than ALS.

\eat{

\vspace{-0.12in}
\paragraph{Robustness} 
In this experiment, we examine the robustness of the SDP relaxation and ALS to noise in the input tensor.
We fix $r = 5$ and $c = 5$.
For each training data point, we add a Gaussian perturbation with zero mean
and variance equal to $\alpha^2 \times \sum_{(i,j,k) \in \Omega} T_{i,j,k,}^2 / m$.
Figure \ref{fig_robust} shows the test error of SDP and ALS when each training
data is corrupted with an $\alpha$ amount of noise.
Even with $\alpha = 0.2$, both methods are still able to learn the underlying
tensor with an error less than $3\%$.
While the performance of ALS is consistently better than SDP relaxations,
we suspect that this is because we formulated the SDP by fitting the training
examples exactly.
It is possible that by using a different formulation, one can obtain similar
results to ALS.
For example, one can consider minimizing the mean squared error,
while constraining the trace to be less than a desired value (see e.g. \cite{YUTC17}).
Interestingly, even when $\alpha = 0.6$, ALS can still recover $T$ with test
error around $5\%$, showing that it is extremely noise-robust.

\begin{figure}[!ht]
	\centering
	\includegraphics[width= 0.43\textwidth]{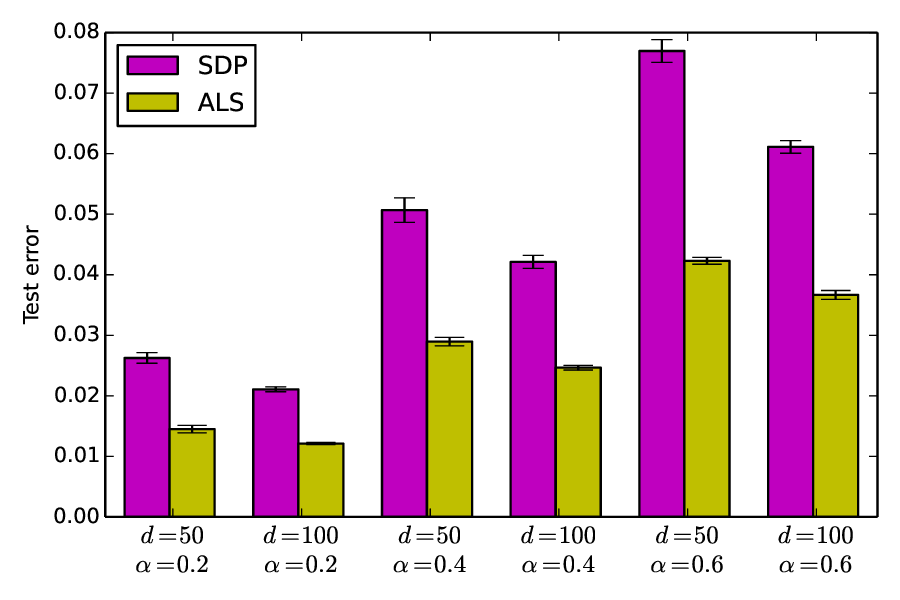}
	\caption{Test error of SDP and ALS when each training data is corrupted
	independently by a Gaussian noise with zero mean and variance
	$\alpha^2 \times {\sum_{(i,j,k) \in \Omega} T_{i,j,k,}^2 / m}$. Both algorithms achieve a high a high degree of noise robustness, and get low test error even when $\alpha=0.6$.}
	\label{fig_robust}
\end{figure}
}

\eat{
\vspace{-0.12in}
\paragraph{Evidence for Conjecture \ref{conj_sample}}
In this experiment, we verify Conjecture \ref{conj_sample} for synthetic data.
Table \ref{table_conj} contains the incoherence of the optimum SDP solution.
We observe that the incoherence does not increase as $d$ increases; moreover we also observe that the sum of singular values other than the top $3r$ singular values
is close zero for all the experiments, providing evidence for Conjecture \ref{conj_sample}.
If we further orthogonalize the underlying random factors,
then we observe that the SDP can recover the original factors up to rotations.
For random factors, we do not always recover the original factors
in the experiment.
For example, this happens if the factors in group $A$ have bigger norms
compared to the factors in group $B$ and $C$.

\begin{table}[!ht]
	\centering
	\begin{tabular}{|c|c|}
	\hline
	dimension & incoherence \\
	$d = 200$ & $1.91 \pm 0.23$ \\
	$d = 100$ & $2.06 \pm 0.63$ \\
	$d = 50 $ & $1.88 \pm 0.54$ \\
	\hline
	\end{tabular}
	\caption{Evidence for Conjecture \ref{conj_sample} on random data: incoherence does not increase as $d$ increases.}
	\label{table_conj}
\end{table}
}

%\vspace{-8pt}
\subsection{Movie Ratings Prediction}\label{sec:movielens}
%\vspace{-6pt}
The Movielens-10M dataset\footnote{{https://grouplens.org/datasets/movielens/10m/}} contains about 10 million ratings (each between 0-5) given by $71,567$ users to $10,681$ movies, along with time stamps for each rating. We test both CP decomposition and the quadratic model on a tensor completion task of predicting missing ratings given a subset of the ratings. We also compare with a matrix factorization-based method, which ignores the temporal information to evaluate if the temporal information in the time stamps is useful.

%\vspace{-0.12in}
\textbf{Methodology.} We split the ratings into a training and test set with two different sampling rates: $p=0.2$ and $p=0.8$ corresponding to $20\%$ and $80\%$ of the entries being in the training set, respectively, and repeat the experiment thrice for each $p$. The lower $p=0.2$ sampling rate is to evaluate the algorithm's performance given very little data. 

To construct the tensor of ratings, we bin the time window into 20-week-long intervals, which gives a tensor of size $(71,567 \times10,681 \times 37)$, where the third mode is the temporal mode. We then use CP and the quadratic model, both with $\ell_2$ regularization, to predict the missing ratings.

For the matrix method, we run matrix factorization with $\ell_2$ regularization on the $(71,567 \times10,681)$ dimensional matrix of ratings. We use alternating minimization with random initialization and tune the regularization parameter for all algorithms. The evaluation metric is the mean squared error (MSE) on the test entries.

%\vspace{-0.12in}
\textbf{Results.} The means and standard deviations of the MSE are reported in Table \ref{table:movie}. There are two key takeaways. Firstly, we can see that the quadratic model consistently yields superior performance than the CP model for the choices of rank\footnote{We found that going to higher rank did not improve the performance of either model.} and sampling rate we explored.

The difference between the performances is also larger for the regime with the lower sampling rate, and we hypothesize that this is due to the superior generalization ability of the quadratic model compared with the CP model. Another reason for the performance gap could be that the tensor is not a low-rank CP tensor since every user only rates a movie once.

The quadratic model also gets a $4\%$ improvement over the baseline, which ignores the temporal information in the ratings and uses matrix factorization.
This is expected---as a user's like or dislike for a genre of movies or a movie's rating may change over time.

\section{Conclusions and Future Work}\label{sec_conc}
In this work, we showed that for a natural non-convex formulation, all local minima are global minima and can be used to recover quadratic tensors using a linear number of samples.
The techniques are also used to show that convex relaxation methods recover quadratic tensors provided with linear samples.
We experimented with a diverse set of real-world datasets, showing that the quadratic model outperforms the CP model when the number of observations is limited.

There are several immediate open questions.
Firstly, is it possible to show a convergence guarantee with a small number of iterations? Secondly, is it possible to achieve similar results  to Theorem \ref{thm_main} with rank $\tilde{O}(\poly(r))$ as opposed to $\Theta(\sqrt{m})$? We believe that solving this may require novel techniques.

\section*{Acknowledgement}

Thanks to Nicolas Boumal for sending us detailed comments that helped improve our work.
We are grateful to the anonymous reviewers for their insightful comments on our work.

\balance
\bibliographystyle{abbrvnat}
\bibliography{rf,rf_prev,vatsal_references,tensor}

\appendix
\onecolumn
\section{Proofs of Lemma \ref{lem_rade} and Theorem \ref{thm_main}}\label{sec_proof_rade}

In this section, we fill in the missing proofs for Theorem \ref{thm_main}.
We present the proof of Lemma \ref{lem_rade}, which bounds the Rademacher complexity of $\cT$, the set of quadratic tensors.

\begin{proof}[Proof of Lemma \ref{lem_rade}]
	Let $\Omega'$ denote a set of $m$ independent samples from $[d]^3$.
	Clearly, we have $\exarg{}{\cL_{\Omega'}(X)} = \cL_{[d]^3}(X)$.
	Hence,
	\begin{align}
		\exarg{\Omega}{\sup_{X \in \cG} \abs{\cL_{\Omega}(X) - \cL_{[d]^3}(X)}}
		=& \exarg{\Omega}{\sup_{X \in \cG} \abs{\frac 1 m \sum_{t \in \Omega} {\inner{A_t}{X - X^{\star}}}^2  - \exarg{\Omega'}{\frac 1 m \sum_{t' \in \Omega'} {\inner{A_{t'}}{X - X^{\star}}}^2 } }} \nonumber \\
		\le & \exarg{\Omega, \Omega'}{\sup_{X \in \cG} \abs{\frac 1 m \sum_{t \in \Omega} {\inner{A_t}{X - X^{\star}}}^2 - \frac 1 m \sum_{t' \in \Omega'} {\inner{A_{t'}}{X - X^{\star}}}^2 }},
		\label{eq_jensen}
	\end{align}
	by the concavity of the supreme operation and the square function.
	Let $\set{\sigma_i}_{i=1}^m$ denote $m$ i.i.d. Rademacher random variables.
	Denote by $\Omega = \set{t_l}_{l=1}^m$ and $\Omega' = \set{t_{l'}}_{l' = 1}^m$.
	By the symmetry of $\Omega$ and $\Omega'$, Equation \eqref{eq_jensen} is equal to:
	\begin{align}
		& \exarg{\Omega, \Omega', \sigma}{\sup_{X \in \cG} \abs{\frac 1 m \sum_{l=1}^m \sigma_l \times \bigbrace{{\inner{A_{t_l}}{X - X^{\star}}}^2 - {\inner{A_{t'_l}}{X - X^{\star}}}^2 } }} \nonumber \\
		& \le	2 \times \exarg{\Omega, \sigma}{\sup_{X \in \cG} \abs{\frac 1 m \sum_{l=1}^m \sigma_l {\inner{A_{t_l}}{X - X^{\star}}}^2 }} \tag{by symmetry between $\Omega$ and $\Omega'$} \nonumber \\
		& = 2 \times \exarg{\Omega, \sigma}{\sup_{T \in \cT'} \abs{\frac 1 m \sum_{l=1}^m \sigma_l T_{i_l, j_l, k_l}^2 }}, \label{eq_spec}
%		& = 2 \times \exarg{\Omega, \sigma}{\sup_{X \in \cG} \abs{\frac 1 m \sum_{l=1}^m \sigma_l \inner{A_{t_l}}{X - X^{\star}}}} \nonumber \\
%		& \le \frac 2 m \times \bigbrace{\exarg{\Omega, \sigma}{\sup_{X \in \cG}\inner{\sum_{l=1}^m \sigma_{t_l} A_{t_l}}{X}}
%		+ \exarg{\Omega, \sigma}{ \inner{\sum_{l=1}^m \sigma_{t_l} A_{t_l}}{X^{\star}}}}. \label{eq_spec}
	\end{align}
	where $\cT' = \set{T - T^{\star} : T \in \cT}$.
	By our assumption on $T^{\star}$, we have that the $\norm{T^{\star}}_{\infty} \lesssim \alpha$.
	Since we also know that $\norm{T}_{\infty} \le 2\alpha$, for any $T \in \cT$.
	Therefore, we have that $\norm{T}_{\infty} \lesssim \alpha$, for any $T \in \cT'$, and the function $f(x) = x^2$ is $O(\alpha)$-Lipschitz, when $\abs{x} \lesssim \alpha$.
	By the contraction principle (Theorem 4.12 in Ledoux and Talagrand \cite{LT13}),
	Equation \eqref{eq_spec} is at most:
	\begin{align*}
		\exarg{\Omega, \sigma}{\sup_{T \in \cT'} \abs{\sum_{l=1}^m \sigma_l T_{i_l, j_l, k_l}^2 }}
		\lesssim& \frac{\alpha} m \times \exarg{\Omega, \sigma}{\sup_{T \in \cT'} \abs{\sum_{l=1}^m \sigma_l T_{i_l, j_l, k_l}}} \\
		=& \frac {\alpha} m \times \exarg{\Omega, \sigma}{\sup_{X \in \cG} \abs{\sum_{l=1}^m \sigma_l \inner{A_{t_l}}{X - X^{\star}} }} \\
		=& \frac {\alpha} m \exarg{\Omega, \sigma}{\sup_{X\in\cG} \abs { \inner{\sum_{l=1}^m \sigma_{t_l} A_{t_l}}{X - X^{\star}} }} \\
		\lesssim & \frac{d\alpha^2} m \exarg{\Omega, \sigma}{\bignorm{\sum_{l=1}^m \sigma_{t_l} A_{t_l}} } \tag{because $\tr(X) \le 6d\alpha$ and $\tr(X^{\star}) \le 3d\alpha$}
	\end{align*}
	To handle the above expectation, we will use the following fact (c.f. Lemma 1 in Davenport et al. \cite{DPVW14} and the proof therein).
	\begin{fact}\label{fact}
		Let $\Omega = \set{(i_1, j_1), (i_2, j_2), \dots, (i_m, j_m)}$ be a set of $m$ uniformly random samples from a $d$ by $d$ matrix.
		Let $Z_{k}$ be the indicator matrix for $(i_k, j_k)$, in other words, the $(i_k, j_k)$-th entry of $Z$ is 1, and 0 otherwise.
		Let $\set{\sigma_{k}}_{k=1}^m$ denote Rademacher random variables.
		We have that
		\[ \exarg{\Omega,\sigma}{\bignorm{\sum_{k=1}^m \sigma_k Z_k}} \lesssim \sqrt{\frac m d + \log d}. \]
	\end{fact}

	To see how to use the above fact in our setting,
	observe that $A_t$ contains nine nonzero entries, for every $t \in [d]^3$.
	If we divide $A_t$ into the $d$ by $d$ submatrices, then there is exactly one nonzero entry in each submatrix with a fixed value.
	Hence we can use Fact \ref{fact} to bound the contribution of each $d$ by $d$ submatrix.
	\footnote{For diagonal blocks, similar results to Fact \ref{fact} can be obtained based on the proof in Lemma 1 of Davenport et al. \cite{DPVW14} (details omitted).}
	Overall, we obtain:
	\[ \exarg{\Omega, \sigma}{\bignorm{\sum_{l=1}^m \sigma_{t_l} A_{t_l}}} \lesssim \sqrt{\frac m d + \log d}. \]
	Combined with Equation \ref{eq_jensen} and \ref{eq_spec}, we obtain the desired conclusion.
	Hence the proof is complete.
\end{proof}

Based on the above Lemma, we can prove Theorem \ref{thm_main}.

\begin{proof}[Proof of Theorem \ref{thm_main}]
	By Lemma \ref{lem_local}, we have that as long as $U$ is a local minimum of $f(\cdot)$, then it is a global minimum.
	In particular, this implies that
	\[ f(U) \le f(U^{\star}) \le (\lambda_1 + \norm{C}) \normFro{U^{\star}}^2 \le 2\lambda_1 \normFro{U^{\star}}, \]
	since $\norm{C} \le \lambda_1$.
	Recall that $\hat{T}$ is the reconstructed tensor.
	By setting $\lambda_1$ to be $\alpha / \sqrt{dm}$, we get that
	\[ \frac 1 m \sum_{(i,j,k) \in \Omega} (\hat{T}_{i,j,k} - T^{\star}_{i,j,k})^2 \le 2\lambda_1 \normFro{U^{\star}}^2 \lesssim \lambda_1 d\alpha \le \alpha^2 \sqrt{\frac d m}, \]
	because $\normFro{U^{\star}}^2 \le 3d\alpha$.
	%By Cauchy-Schwarz inequality, this implies:
	%\begin{align*}
	%\frac 1 m \sum_{(i,j,k) \in \Omega} \abs{\hat{T}_{i,j,k} - T^{\star}_{i,j,k}} \le \sqrt{\lambda_1 d \alpha}
	%\lesssim d \alpha^2 \varepsilon,
	%\end{align*}
	%by setting $\lambda_1 = c^2 \norm{K}_1^2 d^2 \alpha^3 / m$
	%for a fixed constant $c$.

	Next, it is not hard to see that $\norm{e_i^{\top} U} \le \sqrt{2 \alpha}$ by setting $\lambda_2 = {2d} \lambda_1 / {\alpha}$. % otherwise the value of $g(V)$ will be greater than $2\lambda_1 d \alpha$.
	Hence $\inner{UU^{\top}}{e_i e_i^{\top}}$ is at most $2\alpha$ and $\tr(UU^{\top}) \le 6d\alpha$.
	This implies that $UU^{\top} \in \cG$.
	By Lemma \ref{lem_rade}, the Rademacher complexity of all quadratic tensors in $\cT$ is bounded by $O(\alpha^2 \varepsilon)$, recalling that $m \ge d (\log d) / \varepsilon^2$.
	To summarize, we have that the MSE of $\hat{T}$ on $\Omega$ is less than $O(\alpha^2 \sqrt{d / m}) \lesssim \alpha^2 \varepsilon$ and the Rademacher complexity is at most $O(\alpha^2 \varepsilon)$.
	Hence the MSE of $\hat{T}$ on $[d]^3$ can be bounded by:
	\[ \bigo{\alpha^2 \bigbrace{\varepsilon + \sqrt{\frac {\log{\frac 1 {\delta}}} {2m}}} } \lesssim \alpha^2 \varepsilon, \]
	with probability at most $1 - \delta$, over the randomness of $\Omega$ (See e.g. Bartlett and Mendelson \cite{BM02} for more details).
	We can obtain the desired conclusion by setting a small value of $\delta$ (e.g., $1/d$ suffices).
\end{proof}

\paragraph{Limitations of the Quadratic Model.}
In general, there exist tensors that can not be factorized exactly by any quadratic model.
Because if a tensor can be factorized using a quadratic model, then $T$ can be written as the sum of, at most, $O(d)$ rank 1 tensor.
To see this, consider the pairwise tensor model as an example -- the same analysis can be applied to other quadratic models as well.
Given three factors $x \in \dimMatrix{d_1}, y \in \dimMatrix{d_2}$ and $z \in \dimMatrix{d_3}$,
it is not hard to see that the pairwise model  defines the following tensor:
\[T(x, y, z) =  x \otimes y \otimes e + x \otimes e \otimes z + e \otimes y \otimes z,\]
where $e \subseteq \real^d$ denotes the all one vector.
Hence any tensor inside the span of $\{T(x, y, z): x, y, z \subseteq \real^d \}$ can be factorized into at most $3d$ rank one tensor. This lack of representational power can lead to the quadratic model performing worse than the CP model on certain tasks which require high representation ability---and we observe this on a hyperspectral image completion task.

\section{Additional Experiments}\label{sec_add_exp}

\subsection{Convergence Rate of ALS}

In Figure \ref{fig_scale}, we show that ALS can actually converge given a small number of iterations--- we observe that within 30 iterations (each iteration requires solving a sparse $d^2$ by $d$ least squares problems),
ALS can achieve low test error.

\begin{figure}[!ht]
	\centering
	\includegraphics[width= 0.43\textwidth]{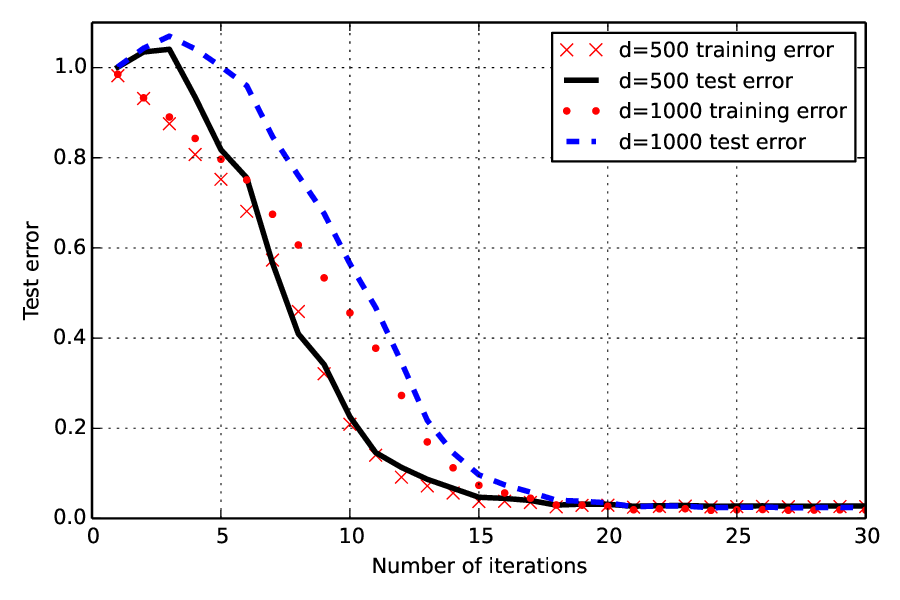}
	\caption{Training and test error for ALS vs the number of iterations. ALS achieves low test error within 30 iterations.}
	\label{fig_scale}
\end{figure}

\subsection{Sample Complexity for Gradient Descent}\label{sec:gd_exp}

We also repeat the same experiment for gradient descent. We run gradient descent with rank $r=d$ for 20000 iterations. Recall that the number of samples $m = c \times d \times r$, and $r = 5$. Figure \ref{fig_accuracy_sdp} shows that the sample complexity of gradient
descent is between $5dr$ and $10dr$ samples. Our experiments suggest that the constants for the sample complexity are slightly better for ALS as compared to gradient descent, and ALS also seems to converge faster to a solution with small error.

\begin{figure}[!ht]
	\centering
	\includegraphics[width = 0.45\textwidth]{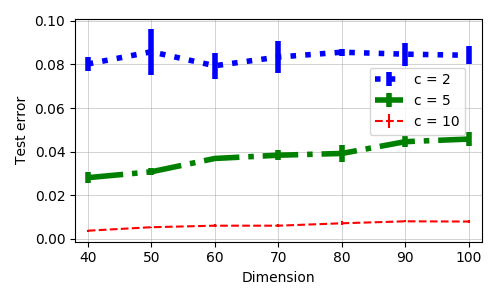}
	\caption{Gradient descent requires about $10dr$ samples to recover a quadratic tensor with random factors, providing evidence that the sample complexity for gradient descent is $O(d)$. Here rank $r = 5$ and number of samples is $m = cdr$.}
	\label{fig_accuracy_sdp}
\end{figure}
\subsection{Recovering Hyperspectral Images}

Since the quadratic model is a special case of the CP model, in principle, it cannot represent any tensor.\footnote{Given three factors $x \in \dimMatrix{d_1}, y \in \dimMatrix{d_2}$ and $z \in \dimMatrix{d_3}$,
the pairwise model defines the following tensor:
$T(x, y, z) =  x \otimes y \otimes e + x \otimes e \otimes z + e \otimes y \otimes z,$
where $e \subseteq \real^d$ denotes the all one vector.
Hence any tensor inside the span of $\{T(x, y, z) : x, y, z \subseteq \real^d \}$ can be factorized into at most $3d$ rank one tensors.}
With enough observations, the quadratic model may not perform as well as the CP model due to limited representational power.
On the other hand, when the amount of observations is limited, the quadratic model still outperforms the CP model.
We describe such an example for the task of completing a hyperspectral image.

We consider recovering a hyperspectral image, ``Riberia'' \cite{foster2006frequency}, which has previously been considered in tensor factorization.
The image is a $1017 \times 1340 \times 33$ tensor $T$, where each image slice corresponds to the same scene being imaged at a different wavelength.

\begin{table*}[th]
	\centering
	\begin{tabular}{c c c}
		\toprule
		Percentage of samples & CP model & Quadratic model \\
		\toprule
		$0.1\%$   & 1.064  & 0.488\\
		$0.3\%$ & 0.495 & 0.424  \\
		$1\%$ & 0.358 & 0.353 \\
		$10\%$   & 0.116  & 0.216\\
		\bottomrule
	\end{tabular}
	\caption{Results for completing the hyperspectral image Riberia.
	We report the test error by taking the median result over three runs in each experiment.}	\label{image}
\end{table*}

We resize the image to $203 \times 268 \times 33$ by downsampling.
We obtain a fraction of sampled entries of the tensor, and the task is to estimate the remaining entries.
We fix the rank of CP and quadratic models to be $r = 100$, measured in terms of the normalized Frobenius error of the recovered tensor $\hat{T}$ on the missing entries (c.f. Equation \eqref{eq_rmse}).
We observe no improvement using even higher ranks for both models in our experiments.
We vary the fraction of samples and compare the performance of the CP model and the quadratic model, and tune the regularization parameter to achieve the best performance for both models.
The results are reported in Table \ref{image}.

We see that the performance of the CP model and the quadratic model vary depending on the fraction of samples available.
While the CP model achieves the best results with 10\% samples, the quadratic model outperforms the CP model when the number of samples is less than 1\%.
For the most parsimonious setting with only $0.1\% \approx 3.6 \times (203 + 268 + 33)$ samples, the quadratic model incurs less than half the RMSE compared to CP.
%\vspace{-8pt}
\subsection{Learning Word Embeddings}\label{sec:word_em}
%\vspace{-6pt}

Word embeddings are vector representations of words, where the vectors and their geometry encodes both syntactic and semantic information. %
We construct word embeddings using the factors obtained using tensor factorization on a suitably normalized tensor of word tri-occurrences. We compare the quality of word embeddings learned by the quadratic model and CP decomposition. This experiment tests if the quadratic model returns meaningful factors, in addition, to accurately predicting the missing entries.

%\vspace{-0.13in}
\paragraph{Methodology.} We construct a $2000$ dimensional cubic tensor $T$ of word tri-occurrences of the 2000 most frequent words in English by using a sliding window of length 3 on a 1.5 billion word Wikipedia corpus, hence the entry $T_{ijk}$ of the tensor is the number of times word $i$, $j$ and $k$ occur in a window of length 3. As in previous work \cite{pennington2014glove}, we construct a normalized tensor $\tilde{T}$ by applying an element-wise nonlinearity of $\tilde{T}_{ijk}=\log( 1+ T_{ijk})$ for each entry of  $T$. %
We then find the factors $\{A, B, C\}$ for a rank 100 factorization of $\tilde{T}$ for the quadratic model and CP decomposition using ALS. The embedding for the $i$th word is obtained by concatenating the $i$th rows of $A$, $B$, and $C$ and then normalizing each row to have a unit norm.

%\vspace{-0.13in}
\paragraph{Evaluation.} In addition to reporting the MSE, we evaluate the learned embeddings on  standard word analogy and similarity tasks. %The analogy tasks evaluate the percentage of word analogy questions that can be solved using the embeddings. The similarity tasks measure the correlation between word similarity scores determined from the embeddings and the true similarity scores. %More details about these tasks can be found in Section \ref{sec:tasks} of the Appendix.
%We evaluate the word embeddings on standard word analogy and word similarity tasks. 
The word analogy tasks  consist of analogy questions of the form \emph{``cat is to kitten as dog is to $\rule{0.4cm}{0.15mm}$?''}, and can be answered by doing simple vector arithmetic on the word vectors. For example, to answer this particular analogy, we take the vector for cat, subtract the vector for kitten, add the vector for dog, and then find the word with the closest vector to the resulting vector. Hence the analogy task tests how much the geometry in the vector space encodes meaningful syntactic and semantic information. 
There are two standard datasets for analogy questions, one of which has more syntactic analogies,  and the other has more semantic analogies. The metric here is the percentage of analogy questions that the algorithm gets correct. The other task we test is a word similarity task where the goal is to evaluate how semantically similar two words are, and this is done by taking the cosine similarity of the word vectors. The evaluation metric is the correlation between the similarity scores assigned by the algorithm and the similarity scores assigned by humans.

%\vspace{-0.13in}
\paragraph{Results.} The results are shown in Table \ref{tasks}. The quadratic model significantly outperforms the CP model on both the MSE metric and the NLP tasks, directly evaluating the embeddings. \eat{In order to have a closer look at the analogy questions at which the quadratic model is better at answering than the CP model, we also evaluated the word embeddings on subsets of analogy questions that test specific analogies. The results are reported in Table \ref{specific_tasks}. The promising performance on the Country/State capitals task indicates the quadratic model may be better at capturing relations for rarer words for which less data is observed due to its superior generalization capabilities. However, properly evaluating this would require testing on a larger dictionary than 2000 words. Our results indicate a significant improvement over the CP model and that a promising direction of future work is to use the quadratic model in conjunction with other algorithms and models for learning word embeddings using tensors to yield higher-quality word embeddings for downstream NLP tasks.}

\begin{table*}[t]
	\centering	
	\begin{tabular}{c c c} 
		\toprule
		Metric & CP model & Quadratic model \\
		\toprule
		MSE & 0.5893 & 0.4253\\
		Syntactic analogy & 30.61\%    & 46.14\%    \\
		Semantic analogy  & 42.37\% & 54.76\%  \\
		Word similarity   & 0.51  & 0.60  \\
		\bottomrule
	\end{tabular}
	\caption{Results for word embedding experiments. The quadratic model significantly outperforms the CP model across all tasks.}	\label{tasks}
\end{table*}

%\vspace{-8pt}
%\subsection{Discussions}
%\vspace{-8pt}

%The quadratic model is a simplification and special case of the CP model and hence has lesser representational power. This can lead to worse performance in certain tensor completion tasks,  we discuss this more with an example of a hyperspectral image completion task in Section \ref{sec:limit}.

\end{document}